\documentclass[a4paper,fleqn]{cas-sc}
\usepackage[numbers,square,sort&compress]{natbib}
\usepackage{graphicx}
\usepackage{multicol,multirow}
\usepackage{amsmath,amssymb,amsthm}
\usepackage[utf8]{inputenc}
\usepackage[T1]{fontenc}
\usepackage[australian]{babel}
\usepackage{tabularx}
\usepackage{algorithmicx}
\usepackage{algpseudocode}
\usepackage{subfigure}
\usepackage{url}
\usepackage{ulem}
\usepackage{cancel}

\usepackage{xcolor}

\usepackage{verbatim}

\renewcommand{\geq}{\geqslant}
\renewcommand{\leq}{\leqslant}
\renewcommand{\ge}{\geqslant}
\renewcommand{\le}{\leqslant}

\newcommand{\vect}[1]{{\mathbf{#1}}}

\renewcommand{\emph}{\textit}

\DeclareMathOperator*{\argmin}{arg\,min}

\definecolor{blue2b}{rgb}{0,0.1,0.3}
\definecolor{blue2}{rgb}{0,0.2,0.7}
\definecolor{red2}{rgb}{0.6,0.1,0.0}
\definecolor{green2}{rgb}{0.1,0.4,0.0}
\definecolor{yel2}{rgb}{0.3,0.2,0.0}
\definecolor{purple2}{rgb}{0.5,0.0,0.5}
\definecolor{blue3}{rgb}{0.65,0.85,1.0}
\definecolor{red3}{rgb}{1.0,0.7,0.5}
\definecolor{green3}{rgb}{0.8,1.0,0.7}
\definecolor{yel3}{rgb}{1.0,1.0,0.7}
\definecolor{grey3}{rgb}{0.95,0.95,0.95}
\definecolor{gray3}{rgb}{0.95,0.95,0.95}

\definecolor{grey1}{rgb}{0.30,0.30,0.30}
\definecolor{grey2}{rgb}{0.20,0.20,0.20}
\definecolor{grey0}{rgb}{0,0,0}

\def\tsc#1{\csdef{#1}{\textsc{\lowercase{#1}}\xspace}}
\tsc{WGM}
\tsc{QE}
\tsc{EP}
\tsc{PMS}
\tsc{BEC}
\tsc{DE}

\makeatletter
\newtheoremstyle{definition}
{3ex}%
{3ex}%
{\upshape}%
{}%
{\bfseries}%
{.}%
{.5em}%
{\thmname{#1}\thmnumber{ #2}\thmnote{ (#3)}}
\makeatother

\theoremstyle{definition}
\newtheorem{theorem}{Theorem}
\newtheorem{lemma}[theorem]{Lemma}
\newtheorem{example}[theorem]{Example}

\newtheorem{corollary}[theorem]{Corollary}
\newtheorem{definition}[theorem]{Definition}
\newtheorem{remark}[theorem]{Remark}

\ExplSyntaxOn
\keys_set:nn { stm / mktitle } { nologo }
\ExplSyntaxOff

\begin{document}
\let\WriteBookmarks\relax
\def\floatpagepagefraction{1}
\def\textpagefraction{.001}

\shorttitle{Hierarchical Clustering with OWA-based Linkages}
\title[mode = title]{Hierarchical Clustering with OWA-based Linkages, the~Lance--Williams Formula, and~Dendrogram~Inversions}

\shortauthors{Gagolewski, Cena, James, Beliakov}

\author[1,2,3]{Marek Gagolewski}[orcid=0000-0003-0637-6028]
\cormark[1]
\cortext[cor1]{Corresponding author}
\ead{m.gagolewski@deakin.edu.au}
\ead[url]{https://www.gagolewski.com}
\credit{Conceptualisation, Methodology, Formal analysis, Writing -- Original Draft}

\author[2]{Anna Cena}
\ead{anna.cena@pw.edu.pl}
\credit{Formal analysis, Visualisation, Writing -- Original Draft}  %

\author[1]{Simon James}
\ead{s.james@deakin.edu.au}
\credit{Formal analysis,Writing -- Original Draft}

\author[1]{Gleb Beliakov}
\ead{gleb@deakin.edu.au}
\credit{Formal analysis,Writing -- Original Draft}

\address[1]{Deakin University, School of IT, Geelong, VIC 3220, Australia}

\address[2]{Warsaw University of Technology,
Faculty of Mathematics and Information Science,
ul. Koszykowa 75, 00-662 Warsaw, Poland}

\address[3]{Systems Research Institute, Polish Academy of Sciences,
ul. Newelska 6, 01-447 Warsaw, Poland}

\begin{abstract}
Agglomerative hierarchical clustering based on Ordered Weighted Averaging (OWA) operators not only generalises the single, complete, and average linkages, but also includes intercluster distances based on a few nearest or farthest neighbours, trimmed and winsorised means of pairwise point similarities, amongst many others. We explore the relationships between the famous Lance--Williams update formula and the extended OWA-based linkages with weights generated via infinite coefficient sequences. Furthermore, we provide some conditions for the weight generators to guarantee the resulting dendrograms to be free from unaesthetic inversions.
\end{abstract}


\maketitle

\ignorespaces\noindent\textit{%
Please cite this paper as:
Gagolewski M., Cena A., James S., Beliakov G., Hierarchical clustering with OWA-based linkages, the Lance–Williams formula, and dendrogram inversions, {\normalfont Fuzzy Sets and Systems} \textbf{473}, 108740, 2023, DOI:10.1016/j.fss.2023.108740
}

\section{Introduction}

Cluster analysis
(e.g., \cite{wierzchon_klopotek}) is a statistical and machine learning task whose aim is to discover interesting or otherwise useful partitions of a given dataset in a purely unsupervised way.

Hierarchical agglomerative clustering algorithms (e.g., \cite{Mullner2011:fastclusteralg}) allow for partitioning the datasets for which merely a pairwise distance function (e.g., a metric) is defined. Most importantly, the number of clusters is not set in advance – a whole hierarchy of nested partitions can be generated with ease, and then depicted on a tree-like diagram called a \textit{dendrogram}.

Hierarchical agglomerative clustering revolves around one simple idea:
in each step, we merge the pair of  \textit{closest} clusters.
To measure the proximity between two point sets,
the intracluster distance is defined as an extension
of the point-pairwise distance called a \textit{linkage function}.

For instance, in the single linkage approach, the distance between a cluster pair
is given by the distance between the closest pair of points,
one from the first cluster, the other from the second one.
In complete linkage, we take the farthest-away pair.
In average linkage, we compute the arithmetic mean between all
the pairwise distances.

In Section~\ref{sec:generallinkages}, we recall two wide classes of
linkage functions that generalise these three cases.
The first group consists of linkages
generated by the well-known Lance--Williams formula
\cite{LanceWilliams1967:hierarchicalformula,Milligan1979:psycho}.
The second class considers
convex combinations of ordered pairwise distances between clusters,
i.e., the OWA operators (ordered weighted averages \cite{Yager1988:owa}).
Such OWA-based linkages were introduced by Yager in \cite{Yager2000:OWAlinkage}
(see also \cite{NasibovKandemirCavas2011:OWAlinkage} where they were re-invented)
and include many linkage generators that are not covered by the classical
Lance--Williams setting.

From the practical side, their usefulness has been thoroughly
evaluated in \cite{CenaGagolewski2020:genieowa},
where also some further tweaks were proposed to increase the quality
of the generated results, e.g., by including in Genie correction
for cluster size inequality \cite{GagolewskiETAL2016:genie}.

However, OWA-based linkages have not been studied thoroughly
from the theoretical perspective. This paper aims to fill this gap.

In particular, after recalling the basic definitions in
Section~\ref{Sec:hier} and Section~\ref{sec:generallinkages}, Section~\ref{sec:owalw} characterises the relationship between
the Lance--Williams linkage update formula and OWA linkages
with two different weight generators.
Then, Section~\ref{sec:owadendr} gives some conditions
for the linkages to yield clusterings
which can be represented aesthetically on dendrograms (without the
so-called inversions).
Section~\ref{sec:conclusion} concludes the paper.

\section{Hierarchical Agglomerative Clustering}\label{Sec:hier}

Let $\mathbf{X}\in\mathbb{R}^{n\times d}$
be the input data set comprised of $n$ points in $\mathbb{R}^d$,
where $\boldsymbol{x}_i=(x_{i,1}, \dots, x_{i, d})$ for each $i$.
From now on, we assume that $\|\boldsymbol{x}_i-\boldsymbol{x}_j\|$
is the Euclidean distance between two given points.
However, let us note that most of the results
presented below hold for any set of $n$ objects equipped with a semimetric.

A \emph{$k$-partition} of $\mathbf{X}$, $k\ge 1$, is defined as
$\mathcal{C}=\{C_1, \dots, C_{k}\}$, where
$\emptyset\neq C_u\subseteq\mathbf{X}$, $C_u\cap C_v = \emptyset$ for $u\neq v$ and $\bigcup_{u=1}^{k} C_u = \mathbf{X}$.

\bigskip
Here is the most basic scheme in which the general
\emph{agglomerative hierarchical clustering} procedure can be formalised:

\begin{enumerate}
\item[1.] Start with $\mathcal{C}^{(0)} = \{C_{1}^{(0)}, \dots, C_{n}^{(0)}\}$, $C_{i}^{(0)}=\{\boldsymbol{x}_i\}$, where each cluster is a singleton.
\item[2.] For $j=1, \dots., n-1$:
\begin{enumerate}
\item[2.1.] Select clusters to merge:

    $$
    (u, v) = \underset{(u, v),\ u<v}{\argmin}\,d(C_u^{(j-1)}, C_v^{(j-1)});
    $$

\item[2.2.] Merge clusters $C_u^{(j-1)}$ and $C_v^{(j-1)}$, i.e.:
\begin{enumerate}
\item[2.2.1.] $C^{(j)}_{u}=C^{(j-1)}_{u}\cup C^{(j-1)}_{v}$;
\item[2.2.2.] $C^{(j)}_{i}=C^{(j-1)}_{i}$ for $u\neq i < v$;
\item[2.2.3.] $C^{(j)}_{i}=C^{(j-1)}_{i+1}$ for $i > v$.
\end{enumerate}
\end{enumerate}
\end{enumerate}

Thus, an agglomerative hierarchical clustering algorithm forms
a~hierarchy of nested $(n,n-1,\dots,2,1)$-partitions
$\mathcal{C}^*=\left(\mathcal{C}^{(0)}, \mathcal{C}^{(1)}, \allowbreak\dots,
\mathcal{C}^{(n-1)}\right)$.

\bigskip
In the algorithm above, $d:2^\mathbf{X}\times 2^\mathbf{X}\to[0, \infty]$
denotes a chosen \textit{linkage function},
whose aim is to measure the distance between two point clusters,
fulfilling at least $d(U, U)=0$, $d(U, V)=d(V, U)$, and:

\begin{equation*}
d(\{\boldsymbol{x}_i\}, \{\boldsymbol{x}_j\})
=\|\boldsymbol{x}_i-\boldsymbol{x}_j\|
\end{equation*}

\noindent
for any $U,V\subseteq\mathbf{X}$ and
$\boldsymbol{x}_i,\boldsymbol{x}_j\in\mathbf{X}$.

\bigskip
Classical choices of $d$ include (see, e.g., \cite{Mullner2011:fastclusteralg}):

\begin{itemize}
 \item the single linkage:
    \[
    d_{\mathrm{MIN}}(U, V) = \min_{\boldsymbol{u}\in U, \boldsymbol{v}\in V}
    \|\boldsymbol{u}-\boldsymbol{v}\|,
    \]

 \item the complete linkage:
    \[
    d_{\mathrm{MAX}}(U, V) = \max_{\boldsymbol{u}\in U, \boldsymbol{v}\in V}
    \|\boldsymbol{u}-\boldsymbol{v}\|,
    \]

 \item the average linkage (UPGMA; \textit{Unweighted Pair Group Method with
Arithmetic Mean}):
    \[
    d_{\mathrm{AMean}}(U, V) = \frac{1}{|U||V|}\sum_{\boldsymbol{u}\in U,
\boldsymbol{v}\in V} \|\boldsymbol{u}-\boldsymbol{v}\|,
    \]

    \item the Ward linkage:
\begin{eqnarray*}
  d_{\mathrm{Ward}}(U, V) &=&
  \frac{2}{{|U| + |V|}} \sum_{\boldsymbol{u}\in U, \boldsymbol{v}\in V} \|\boldsymbol{u}-\boldsymbol{v}\|^2\\
 &-&  \frac{|V|}{|U| (|U| + |V|)} \sum_{\boldsymbol{u},\boldsymbol{u}'\in U}{\|\boldsymbol{u}-\boldsymbol{u}'\|^2}\\
 &-& \frac{|U|}{|V| (|U| + |V|)}\sum_{\boldsymbol{v},\boldsymbol{v}'\in V} {\|\boldsymbol{v}-\boldsymbol{v}'\|^2},
\end{eqnarray*}

\item the centroid linkage (UPGMC; \textit{Unweighted Pair
Group Method Centroid}):
\[
d_{\mathrm{Cent}}(U, V) =\|\boldsymbol{\mu}_u-\boldsymbol{\mu}_v\|,
\]

\noindent
where $\boldsymbol{\mu}_u$, $\boldsymbol{\mu}_v$ are the respective
clusters' centroids (componentwise arithmetic means),

\item weighted average linkage (WPGMA; \textit{Weighted Pair Group Method with Arithmetic Mean}):
\[
d_{\mathrm{WAMean}}(U, V) =
\frac{|W|}{|W| + |Z|}d_{\mathrm{WAMean}}(U, W) +
\frac{|Z|}{|W|+|Z|}d_{\mathrm{WAMean}}(U, Z),
\]

\noindent
assuming that $V=W\cup Z$ in one of the previous iterations,

\item the median linkage (WPGMC;~\textit{Weighted Pair
Group Method Centroid}) given by:
\[
d_{\mathrm{Median}}(U, V) = \frac{1}{2}d_{\mathrm{Median}}(U, W)+\frac{1}{2}d_{\mathrm{Median}}(U, Z)-\frac{1}{4}d_{\mathrm{Median}}(W, Z),
\]

\noindent
assuming that $V=W\cup Z$ in one of the previous iterations.
\end{itemize}

Some of the above cases can be generalised through the Lance--Williams
formula \cite{LanceWilliams1967:hierarchicalformula,Milligan1979:psycho}
or the OWA-based linkages \cite{Yager2000:OWAlinkage},
which we shall discuss next.

\section{Linkage Classes}\label{sec:generallinkages}

Let us recall two noteworthy linkage classes, based respectively on
the Lance--Williams formula and the OWA operator-based intercluster
distances.

\subsection{Lance--Williams Linkages}

In \cite{LanceWilliams1967:hierarchicalformula},
G.~Lance and~W.~Williams proposed
an iterative formula that generalises many common linkages
and allows for a fast update of the intercluster distances
after each cluster pair merge.

Assuming that in the $j$-th step of the procedure
we are about to merge $C_u^{(j-1)}$ and $C_v^{(j-1)}$,
then for every other (intact) cluster $C_z^{(j)}=C_z^{(j-1)}$, $z\not\in\{u,v\}$,
the new distances are:
\begin{eqnarray}\label{Eq:LanceWilliams}
d\Big(C_z^{(j)}, C_u^{(j-1)}\cup C_v^{(j-1)}\Big)\nonumber
&=&\alpha_u\, d(C_z^{(j-1)}, C_u^{(j-1)})\\ \nonumber
&+&\alpha_v\, d(C_z^{(j-1)}, C_v^{(j-1)})\\ \nonumber
&+&\beta\phantom{{}_u}\, d(C_u^{(j-1)}, C_v^{(j-1)})\\
&+&\gamma\phantom{{}_u}\,\Big|d(C_z^{(j-1)}, C_u^{(j-1)}) -
d(C_z^{(j-1)}, C_v^{(j-1)})\Big|,
\end{eqnarray}

\noindent
for some $\alpha_u$, $\alpha_v$, $\beta$, and $\gamma$
that might depend on $n_u=|C_u^{(j-1)}|, n_v=|C_v^{(j-1)}|,$ and
$n_z=|C_z^{(j)}|$.

\begin{table}[h!]
\centering\caption{Common coefficients in the Lance--Williams formula
\cite{LanceWilliams1967:hierarchicalformula,Milligan1979:psycho}}\label{Tab:LanceWilliams}
\begin{tabular}{|l|c|c|c|c|}
\hline
{linkage} & $\alpha_u$ & $\alpha_v$
                 & $\beta$ & $\gamma$\\
\hline
\hline
single & $\displaystyle\frac{1}{2}$ & $\displaystyle\frac{1}{2}$ & 0 & $-\displaystyle\frac{1}{2}$ \\[2ex]
complete & $\displaystyle\frac{1}{2}$ & $\displaystyle\frac{1}{2}$ & 0 & $+\displaystyle\frac{1}{2}$ \\[2ex]
average (UPGMA)  & $\displaystyle\frac{n_{u}}{n_{u}+n_{v}}$ & $\displaystyle\frac{n_{v}}{n_{u}+n_{v}}$ & 0 & 0 \\[2ex]
weighted average (WPGMA)  & $\displaystyle\frac{1}{2}$ & $\displaystyle\frac{1}{2}$ & 0 & 0 \\[2ex]
centroid (UPGMC)  & $\displaystyle\frac{n_{u}}{n_{u}+n_{v}}$ & $\displaystyle\frac{n_{v}}{n_{u}+n_{v}}$
                             & $-\displaystyle\frac{n_{u} n_{v}}{(n_{u}+n_{v})^{2}}$ & 0 \\[2ex]
median (WPGMC)  & $\displaystyle\frac{1}{2}$ & $\displaystyle\frac{1}{2}$ & $-\displaystyle\frac{1}{4}$ & 0 \\[2ex]
Ward  & $\displaystyle\frac{n_{u}+n_{z}}{n_{u}+n_{j}+n_{z}}$ & $\displaystyle\frac{n_{v}+n_{z}}{n_{u}+n_{v}+n_{z}}$
                             & $-\displaystyle\frac{n_{z}}{n_{u}+n_{v}+n_{z}}$ & 0 \\[2ex]
\hline
\end{tabular}
\end{table}

Table \ref{Tab:LanceWilliams} gives some common choices of the above coefficients.

Note that the Lance--Williams formula only utilises the information about
$d(C_u^{(j-1)}, C_v^{(j-1)})$,
$d(C_z^{(j-1)}, C_u^{(j-1)})$, and
$d(C_z^{(j-1)}, C_v^{(j-1)})$,
as well as the cardinalities of the clusters.
Further, as we would like the linkage to be symmetric,
it is required that $\alpha_u(n_u, n_v, n_z)=\alpha_v(n_v, n_u, n_z)$,
$\beta(n_u, n_v, n_z)=\beta(n_v, n_u, n_z)$, and
$\gamma(n_u, n_v, n_z)=\gamma(n_v, n_u, n_z)$.

\subsection{OWA-based linkages}

OWA operators, i.e., convex combinations of order statistics,
were introduced in the aggregation and decision making
context by Yager in \cite{Yager1988:owa}.
Let us introduce their version that acts on element sequences
of any length.

\begin{definition}
An \emph{extended} \cite{MayorCalvo1997:eaf,CalvoETAL2000:generationwtri}
\emph{OWA operator} is defined as:
\[
\mathrm{OWA}_{\triangle}(d_1, d_2, \dots, d_m)=\sum_{i=1}^{m} c_{i,m} d_{(i)}
\]

\noindent
where a given {weighting triangle} is
$\triangle=(c_{i,m}\in[0,1],\,m\in\mathbb{N}, i=1,\dots, m:
(\forall m) \sum_{i=1}^{m} c_{i, m}=1)$,
and $d_{(1)} \ge d_{(2)} \ge \dots \ge d_{(m)}$.
\end{definition}

By definition, $\mathrm{OWA}_{\triangle}$ is symmetric,
idempotent, and nondecreasing in each variable;
compare \cite{GrabischETAL2009:aggregationfunctions,BeliakovBustinceCalvo2016,%
Gagolewski2015:datafusionbook}.
Hence, it is also internal, i.e.,
$d_{(1)}\ge\mathrm{OWA}_{\triangle}(d_1,\dots,d_m)\ge d_{(m)}$
for all possible inputs $d_1,\dots,d_m$ of any cardinality.

\bigskip
In the clustering context,
OWA-based linkage was proposed by Yager in \cite{Yager2000:OWAlinkage};
see also \cite{NasibovKandemirCavas2011:OWAlinkage,CenaGagolewski2020:genieowa}.

\begin{definition}
For a given weighting triangle $\triangle$,
the \emph{$\mathrm{OWA}_{\triangle}$-based linkage} is defined as:
\[
d_\triangle(C_u, C_v) = \mathrm{OWA}_{\triangle}(\{
    \|\boldsymbol{u}-\boldsymbol{v}\|:
    \boldsymbol{u}\in C_u, \boldsymbol{v}\in C_v
\})
\]

\noindent
for any point sets $C_u$ and $C_v$,
i.e., it is the OWA operator applied on all the $m=|C_u| |C_v|$
pairwise distances between the two sets.
\end{definition}

\begin{remark}
In particular, for weights fulfilling:

\begin{itemize}
\item $c_{i,m} = \frac{1}{m}$: we obtain the average linkage,
\item $c_{m,m} = 1$ and $c_{i,m}=0$ for $i<m$: we get the single
linkage,
\item $c_{1,m} = 1$ and $c_{i,m}=0$ for  $i>1$: we enjoy the complete linkage.
\end{itemize}

Note that numerous new linkages that do not fit the Lance--Williams formula
are now possible. This includes, e.g., distance aggregates that correspond to the median and any other quantiles, trimmed and winsorised means, the arithmetic mean of a few first smallest observations, a fuzzified/smoothed minimum, and so forth; see \cite{CenaGagolewski2020:genieowa} for many example classes.
\end{remark}

There are a few generic ways to generate
the weighting triangles known in the literature; see \cite{JamisonETAL1965:convwave,Yager1988:owa}
and \cite{GomezETAL:agopmissing2014,BeliakovJames2013:stability}.
In this paper, we will be interested in studying:

\begin{itemize}
 \item $c_{i, m}=\frac{c_i}{\sum_{j=1}^m c_j}$, where
 $(c_1,c_2,\dots)$ is such that $c_i\ge 0$ for all $i=1,2,\dots$ and $c_1=1$,
 \item $c_{m-i+1, m}=\frac{c_{i}}{\sum_{j=1}^m c_j}$, where
 $(c_1,c_2,\dots)$ is such that $c_i\ge 0$ for all $i=1,2,\dots$ and $c_1=1$.
\end{itemize}

\noindent
The assumption that $c_1=1$ is with no loss in generality
provided that we disallow the ill-defined case $c_1=0$,
where we could have utilised the convention $0/0=1$. We do not enable
it to keep the presentation simple.

Hence, for a given coefficients vector $\boldsymbol{c}=(c_1,c_2,\dots)$
with $c_1=1$ and $c_2,c_3,\dots,\ge 0$,
we can define two extended OWA operators:
\begin{equation}
\mathrm{OWA}_{\boldsymbol{c}}(d_1,\dots,d_m)  = \frac{\sum_{i=1}^m c_i d_{(i)}}{\sum_{i=1}^m c_i},
\end{equation}

\begin{equation}
\mathrm{OWA}'_{\boldsymbol{c}}(d_1,\dots,d_m)  = \frac{\sum_{i=1}^m c_i d_{(m-i+1)}}{\sum_{i=1}^m c_i},
\end{equation}

\noindent
for any $m$ and $d_1,\dots,d_m$,
where $d_{(i)}$ denotes the $i$-th greatest value in the sequence.

\section{OWA linkages and the Lance--Williams formula}\label{sec:owalw}

In this section, we characterise which OWA-based linkages
(defined via $\mathrm{OWA}_{\boldsymbol{c}}$ and
$\mathrm{OWA}'_{\boldsymbol{c}}$) can be expressed by the Lance--Williams
formula, and vice versa.
It turns out that under the two assumed weight generation schemes,
these only include the single, average, and complete linkages.

\begin{theorem}\label{Thm:lwowa1}
Assume that $d$ is generated by the Lance--Williams formula,
i.e., for any $C_u, C_v, C_z$ it holds:
\begin{eqnarray*}
d\Big(C_z, C_u\cup C_v\Big)
&=&\alpha_u\,d(C_z, C_u)  \\
&+&\alpha_v\,d(C_z, C_v) \\
&+&\beta\phantom{{}_{u}}\,d(C_u, C_v) \\
&+&\gamma\phantom{{}_{u}}\,\Big|d(C_z, C_u) - d(C_z, C_v)\Big|,
\end{eqnarray*}

\noindent
with
$\alpha_u(|C_u|, |C_v|, |C_z|)=\alpha_v(|C_v|, |C_u|, |C_z|)$,
$\beta(|C_u|, |C_v|, |C_z|)=\beta(|C_v|, |C_u|, |C_z|)$, and
$\gamma(|C_u|, |C_v|, |C_z|)=\gamma(|C_v|, |C_u|, |C_z|)$.
Then there exists $\boldsymbol{c}=(c_1,c_2,\dots)$ with $c_1=1$
and $c_2,c_3\ge 0$ such that for every $C_u, C_v$ it holds:
\[
d(C_u, C_v) = \mathrm{OWA}_{\boldsymbol{c}}(\{ \|\boldsymbol{u}-\boldsymbol{v}\|: \boldsymbol{u}\in C_u, \boldsymbol{v}\in C_v \})
\]

\noindent
if and only if $\boldsymbol{c}$ is either:

\begin{itemize}
\item $\boldsymbol{c}=(1, 0, 0, \dots)$ (complete linkage -- the maximum) or
\item $\boldsymbol{c}=(1, 1, 1, \dots)$ (average linkage -- the arithmetic mean).
\end{itemize}
\end{theorem}

\begin{proof}
That these conditions are sufficient is evident.

Hence, assume that $d$ is generated by the Lance--Williams formula
and that:
\[
d\Big(C_z, C_u\cup C_v\Big) = \frac{\sum_{i=1}^{n+m} c_i w_{(i)}}{\sum_{i=1}^{n+m} c_i}
\]

\noindent
for $\boldsymbol{w}=(u_1,\dots,u_n,v_1,\dots,v_m)$,
where $u_i$ is the distance between an $i$-th pair
of points in $C_u\times C_z$ and $v_j$ is the distance
between a $j$-th pair of points in $C_v\times C_z$.
Due to the symmetry of OWAs, it does not matter how we enumerate the pairs,
hence we assume $u_1\ge\dots\ge u_n$ and $v_1\ge\dots\ge v_m$,
where
$n=|C_u| |C_z|$, $m=|C_v| |C_z|$. %

It is apparent that $d\Big(C_z, C_u\cup C_v\Big)$ cannot depend on
$d(C_u, C_v)$, and hence it necessarily holds that $\beta=0$.
Therefore, let us note that if $\mathrm{OWA}_{\boldsymbol{c}}(u_1,\dots,u_n)\ge
\mathrm{OWA}_{\boldsymbol{c}}(v_1,\allowbreak\dots,v_m)$, then we have:
\[
\mathrm{OWA}_{\boldsymbol{c}}(u_1,\dots,u_n,v_1,\dots,v_m)
= (\alpha_u+\gamma) \mathrm{OWA}_{\boldsymbol{c}}(u_1,\dots,u_n)
+ (\alpha_v-\gamma) \mathrm{OWA}_{\boldsymbol{c}}(v_1,\dots,v_m).
\]

As mentioned earlier, each OWA operator is internal. Hence, it necessarily holds
$\alpha_u+\alpha_v=1$, a condition which we obtain
by considering $n=m=1$ and $u_1=v_1=1$, because it yields:
\[
\frac{c_1 \cdot 1 + c_2 \cdot 1}{c_1+c_2} = (\alpha_u+\gamma)\cdot 1 + (\alpha_v-\gamma)\cdot 1.
\]
From now on assume that $\boldsymbol{c}\neq(1,0,0,\dots)$ (i.e.,
it is not the complete linkage),
i.e., $\alpha_u+\gamma<1$ and $\alpha_v-\gamma>0$.

In the case where $n=m$ and
$u_1=v_1=1$ and $u_2=v_2=u_3=v_3=\dots=0$
we have:
\[
\frac{ (c_1+c_2)\cdot 1 }{\sum_{i=1}^{2n} c_i}
=
\frac{c_1\cdot 1}{\sum_{i=1}^{n} c_i}.
\]

\noindent
Under the assumption that $c_1=1$,
this implies:
\[
c_2  = \frac{\sum_{i=n+1}^{2n} c_i}{\sum_{i=1}^{n} c_i}
\]

\noindent
for all $n\ge 1$.

Now consider the case $n=1$, $m>1$ with $u_1\ge v_1\ge v_2\ge \dots$.
If $u_1=1$ and $v_1=v_2=\dots=0$, this breeds:
\[
\frac{c_1\cdot 1}{\sum_{i=1}^{m+1} c_i} = (\alpha_u+\gamma)\cdot 1.
\]

\noindent
Therefore, $(\alpha_v-\gamma)
=\frac{\sum_{i=2}^{m+1} c_i}{\sum_{i=1}^{m+1} c_i}$.
If $u_1=v_1=1$ and $v_2=v_3=\dots=0$,
then:
\[
\frac{c_1\cdot 1+c_2\cdot 1+c_3\cdot 0 + \dots + c_{m+1}\cdot 0}{\sum_{i=1}^{m+1} c_i} =
\frac{c_1}{\sum_{i=1}^{m+1} c_i} \cdot 1
+\frac{\sum_{i=2}^{m+1} c_i}{\sum_{i=1}^{m+1} c_i} \frac{c_1\cdot 1 + c_2\cdot 0 + \dots + c_t\cdot 0}{\sum_{i=1}^{m} c_i}.
\]

\noindent
Thus,
\[
c_2 = c_1 \frac{\sum_{i=2}^{m+1} c_i}{\sum_{i=1}^{m} c_i}.
\]

\noindent
Assuming $c_1=1$ and $m=2$, this yields:
\[
c_3=c_2^2.
\]

\noindent
With $c_1=1$ and studying further $m>2$, we get
each time that $c_m=c_2^{m-1}$.
But from the previous equation we have that
$c_2 = \frac{c_3+c_4}{c_1+c_2} = \frac{c_2^2 + c_2^3}{1+c_2}$,
Hence, $1+c_2 = c_2+c_2^2$ and thus $c_2=1$,
which implies that necessarily for all $i$ it must hold $c_i=1$,
which corresponds to the average linkage.

As we have already considered the complete linkage case separately,
the proof is complete.
\end{proof}

\begin{theorem}
Assume that $d$ is generated by the Lance--Williams formula,
just like above.
Then there exists $\boldsymbol{c}=(c_1,c_2,\dots)$ with $c_1=1$
and $c_2,c_3\ge 0$ such that for every $C_u, C_v$ we have:
\[
d(C_u, C_v) = \mathrm{OWA}'_{\boldsymbol{c}}(\{ \|\boldsymbol{u}-\boldsymbol{v}\|: \boldsymbol{u}\in C_u, \boldsymbol{v}\in C_v \})
\]

\noindent
if and only if $\boldsymbol{c}$ is either:
\begin{itemize}
\item $\boldsymbol{c}=(1, 0, 0, \dots)$ (single linkage -- the minimum) or
\item $\boldsymbol{c}=(1, 1, 1, \dots)$ (average linkage -- the arithmetic mean).
\end{itemize}
\end{theorem}

\begin{proof}
Sufficiency of the above is obvious. The reasoning
required to show the necessary part is very similar to the one we
have conveyed in the proof of Theorem~\ref{Thm:lwowa1}.

As it necessarily holds that $\beta=0$,  %
let us note that if
$\mathrm{OWA}'_{\boldsymbol{c}}(u_1,\dots,u_n)\ge
\mathrm{OWA}'_{\boldsymbol{c}}(v_1,\allowbreak\dots,v_m)$, then we have:
\[
\mathrm{OWA}'_{\boldsymbol{c}}(u_1,\dots,u_n,v_1,\dots,v_m)
= (\alpha_u+\gamma) \mathrm{OWA}'_{\boldsymbol{c}}(u_1,\dots,u_n)
+ (\alpha_v-\gamma) \mathrm{OWA}'_{\boldsymbol{c}}(v_1,\dots,v_m),
\]

\noindent
for some $\alpha_u+\alpha_v=1$.  %

From now on assume that $\boldsymbol{c}\not\in(1,0,0,\dots)$ (i.e.,
it is not the single linkage),
i.e., $\alpha_u+\gamma>0$ and $\alpha_v-\gamma<1$.  %

Consider the case where $n=m$ and
$u_1=v_1=u_2=v_2=\dots=u_{n-1}=v_{m-1}=1$ and $u_n=v_m=0$.
This implies:
\[
c_2 = \frac{\sum_{i={n+1}}^{2n} c_i}{\sum_{i=1}^n c_i}
\]

\noindent
for all $n\ge 1$.

Next, we study $n>1$ and $m=1$ with $u_1\ge\dots\ge u_n\ge v_1$.
If $v_1=0$ and $u_1=\dots=u_n=1$,
then we obtain:
\[
\frac{\sum_{i=2}^{n+1} c_i}{\sum_{i=1}^{n+1} c_i} = (\alpha_u+\gamma)\cdot 1.
\]

\noindent
Therefore,
$(\alpha_v-\gamma)=
\frac{c_1}{\sum_{i=1}^{n+1} c_i}.
$
If $u_1=1$ and $u_2=\dots=u_n=v_1=0$,
then:
\[
c_{n+1} = c_n \frac{\sum_{i=2}^{n+1} c_i}{\sum_{i=1}^{n} c_i},
\]

\noindent
which, similarly as in the previous proof, implies that
$c_i=1$ for all $i$ (average linkage), QED.
\end{proof}

\section{Inversion-free dendrograms}\label{sec:owadendr}

We can depict a~hierarchy of nested partitions, $\mathcal{C}^*=\{\mathcal{C}^{(0)}, \mathcal{C}^{(1)}, \allowbreak\dots, \mathcal{C}^{(n-1)}\}$, on a \textit{dendrogram}, which is an undirected tree whose leaves represent the initial singleton clusters $C_1^{(0)},\dots,\allowbreak C_n^{(0)}$. The root corresponds to $C_1^{(n-1)}=\mathbf{X}$, and the inside nodes depict the clusters which are merged in each step of the procedure.

\begin{figure}[t!]
\centering
\includegraphics[width=1.0\linewidth]{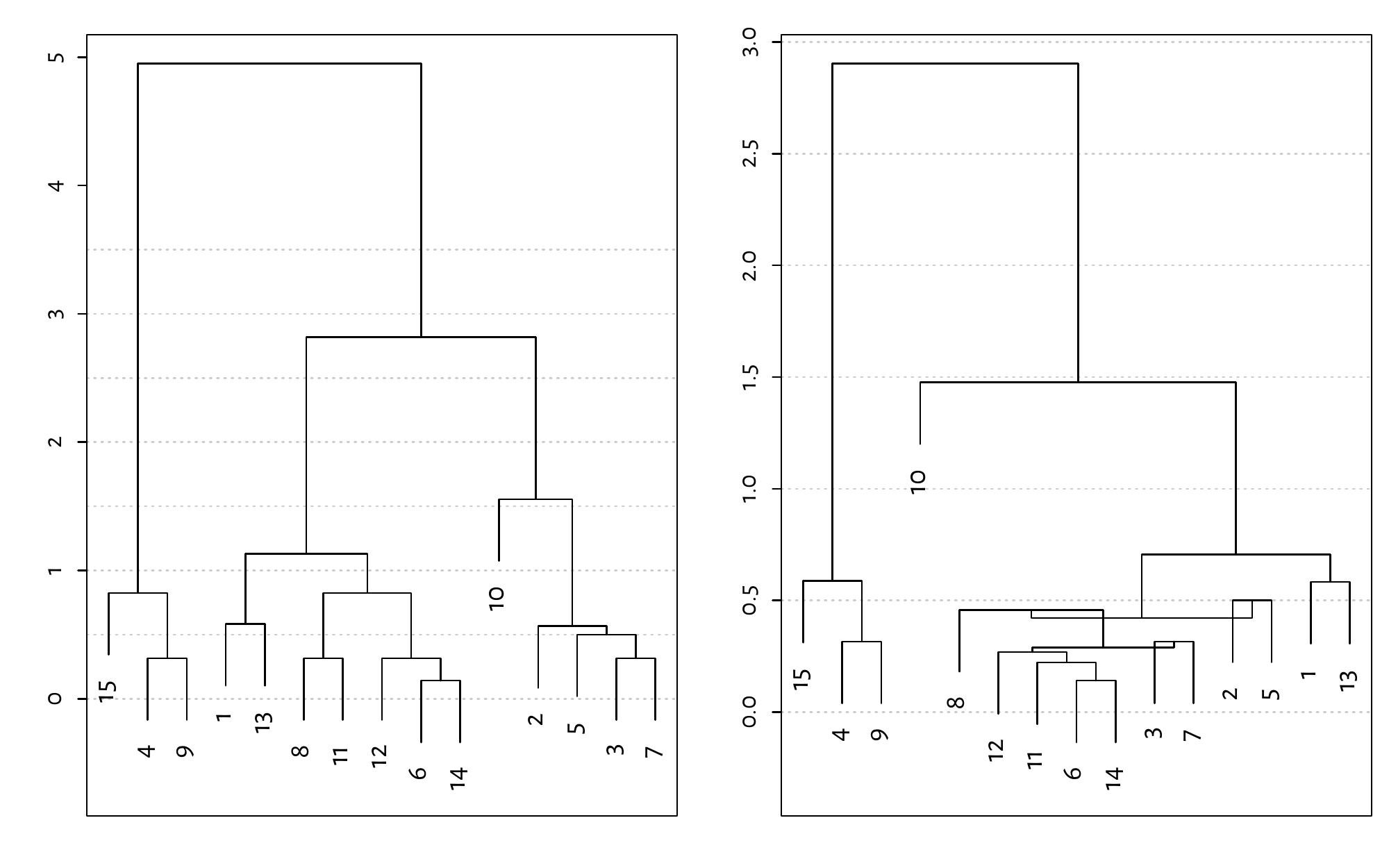}
\caption{\label{fig:dendrograms} Dendrograms for complete (lefthand side) and
centroid (righthand side) linkage-based clustering of the same example subset
of the \texttt{iris} dataset, as plotted by the \texttt{stats:::plot.hclust}
function in R. Note the inversions.}
\end{figure}

Let $h_d:\mathcal{C}^*\to[0, \infty)$
be a function which assigns each hierarchy level $\mathcal{C}^{(j)}$
a specific height,
given by:
\begin{equation*}
h_d(\mathcal{C}^{(j)})=
  \begin{cases}
    0  & \quad j=0,\\
    \min_{u, v} d(C_u^{(j-1)},\allowbreak C_v^{(j-1)}) & \quad j \ge 1.\\
  \end{cases}
\end{equation*}

\noindent
In other words, it is the distance (as given by the chosen linkage $d$)
between the two clusters merged in the $j$-th step.

\begin{remark}
For example, Figure~\ref{fig:dendrograms}
depicts two cluster dendrograms of an example dataset.
On both dendrograms, we see the merging of singleton clusters
based on $d(\{\boldsymbol{x}_6\}, \{\boldsymbol{x}_{14}\})=0.14142$
(in the first stage)
and $d(\{\boldsymbol{x}_1\}, \{\boldsymbol{x}_{13}\})\allowbreak = \allowbreak 0.5831$ (at different stages).
In both cases, the final merge is done between
$\{\boldsymbol{x}_{15},\boldsymbol{x}_4,\boldsymbol{x}_9\}$ and
$\boldsymbol{X}\setminus\{\boldsymbol{x}_{15},\boldsymbol{x}_4,\boldsymbol{x}_9\}$,
however, these are at different heights as specified by the linkage
functions in use (complete on the lefthand side and centroid on the right).
\end{remark}

Unfortunately, as noted in \cite{Johnson1967:ultrametric,Milligan1979:psycho},
the height function $h_d$ is not necessarily increasing for every linkage $d$,
i.e., it does not always hold that:
\begin{equation}\label{Eq:monotonicity}
h_d(\mathcal{C}^{(j-1)}) \le h_d(\mathcal{C}^{(j)})
\end{equation}

\noindent
for all $j$.
This may lead to dendrogram pathologies
known as \textit{inversions}.

\begin{remark}
The right side of Figure~\ref{fig:dendrograms} depicts
the clustering of an example dataset based on the centroid linkage.
Note the two dendrogram inversions, which are due to the fact
that the intercluster distances at some higher levels of the hierarchy
are smaller than the ones computed in some previous merge steps.
This happens when we merge, e.g., the cluster that features the
points $\boldsymbol{x}_3$ and $\boldsymbol{x}_7$
with the one containing $\boldsymbol{x}_6$, $\boldsymbol{x}_{11}$,
$\boldsymbol{x}_{12}$, and $\boldsymbol{x}_{14}$.
In this case, we have $d(\{\boldsymbol{x}_3\}, \{\boldsymbol{x}_7\})>
d(\{\boldsymbol{x}_3, \boldsymbol{x}_7\}, \{ \boldsymbol{x}_6,\boldsymbol{x}_{11},\boldsymbol{x}_{12},\boldsymbol{x}_{14}\})$.
\end{remark}

\begin{theorem}\label{Thm:monotonicity}
$h_d$ is increasing
if and only if
$(\forall j)$ $(\forall z\in\{1, \dots, n-j\})$, $z\neq u$
and
for $u, v\in\{1, \dots, n-j+1\}$ such that
$(u, v) = \argmin_{u, v}d(C_u^{(j-1)}, C_v^{(j-1)})$
it holds:
\begin{equation}\label{Eq:monotonicityHierarchy}
d(C_z^{(j)}, C_u^{(j-1)}\cup C_v^{(j-1)}) \ge
d(C_u^{(j-1)}, C_v^{(j-1)}).
\end{equation}
\end{theorem}

\begin{proof}
$(\Longrightarrow)$ Let us fix $j\in\{1, \dots, n-1\}$ and let $(u, v) = \argmin_{u, v}d(C_u^{(j-1)},\allowbreak C_v^{(j-1)})$.
By the definition of $h_d$, we have
$h_d(\mathcal{C}^{(j)})=d(C_u^{(j-1)}, C_v^{(j-1)})$.
Without loss of generality, we can assume that $u<v$, and therefore
$C^{(j)}_u = C_u^{(j-1)}\cup C_v^{(j-1)}$.
From this it follows that:
\begin{align*}
d(C_u^{(j-1)}, C_v^{(j-1)}) = h_d(\mathcal{C}^{(j)})
&\le h_d(\mathcal{C}^{(j+1)}) = d(C_q^{(j)}, C_w^{(j)})\\
&\le d(C_z^{(j)}, C_u^{(j)}) = d(C_z^{(j)}, C_u^{(j-1)}\cup C_v^{(j-1)})
\end{align*}

\noindent
for all $z\in\{1, \dots, n-j\}$
where $d(C_q^{(j)}, C_w^{(j)}) = \min_{q, w}d(C_q^{(j)}, C_w^{(j)})$.
\bigskip

$(\Longleftarrow)$ Let us assume that:
\[
d(C_z^{(j)}, C_u^{(j-1)}\cup C_v^{(j-1)}) \ge
d(C_u^{(j-1)}, C_v^{(j-1)}) =
\min_{u<v}d(C_u^{(j-1)},\allowbreak C_v^{(j-1)}),
\]

\noindent
but at the same time:
\[
\min_{u<v}d(C_u^{(j-1)},\allowbreak C_v^{(j-1)}) >
\min_{z,w}d(C_z^{(j)},\allowbreak C_w^{(j)}).
\]

\noindent
We thus get $C_w^{(j)}\neq C_u^{(j-1)}\cup C_v^{(j-1)}\neq C_z^{(j)}$
and also $C_w^{(j)} = C_w^{(j-1)}$
and $C_z^{(j)} = C_z^{(j-1)}$.
This contradicts our assumption, since
$\min_{z,w}d(C_z^{(j)},\allowbreak C_w^{(j)}) =
\min_{z,w}d(C_z^{(j-1)},\allowbreak C_w^{(j-1)})$.
Therefore,  $h_d$ is increasing and  the proof is complete.
\end{proof}

\bigskip
As far as the Lance--Williams formula is concerned, we have the following result \cite{Milligan1979:psycho}.

\begin{theorem}\label{Th:MilliganMonotonic}
If:
\begin{enumerate}[(1)]
\item $\alpha_u+\alpha_v+\beta \ge 1$,
\item $\alpha_u\ge0$, $\alpha_v\ge0$
\item $\gamma$ is such that:
\begin{enumerate}
\item $\gamma\ge 0$ or
\item $\gamma<0$ and $|\gamma|\le\min\{\alpha_u, \alpha_v\}$,
\end{enumerate}
\end{enumerate}
then for $d$ given by \eqref{Eq:LanceWilliams},
it holds:
\begin{equation*}\label{Eq:monotonicMilligan}
d(C_z^{(j)}, C_u^{(j-1)}\cup C_v^{(j-1)})\ge d(C_u^{(j-1)}, C_v^{(j-1)}),
\end{equation*}

\noindent
for every $j=1, \dots, n-1$,
$(u,v)=\argmin_{(u, v)}d(C^{(j-1)}_u, C^{(j-1)}_v)$, $u\neq v$,
$z\neq u$, and
$z\in\{1, \dots, n-j\}$,
i.e.,
the corresponding $h_d$ is increasing. %
\end{theorem}

See \cite{Milligan1979:psycho} for a proof.

Hence, single, complete, average, weighted average,
and Ward linkages yield increasing $h_d$.

\bigskip
Let us move on to the OWA linkage case.
We note that already even very simple weighting triangles can
fail to produce increasing OWA-based $h_d$.

\begin{remark}
If
$d_\triangle(d_1,\dots,d_m)=0.5(d_{(m)}+d_{(m-1)})$ for $m>1$,
i.e., the arithmetic mean of the two smallest values,
then $h_d$ is not increasing.
For example, assume that
$\mathbf{X}=\{\boldsymbol{x}_{1}, \boldsymbol{x}_{2}, \boldsymbol{x}_{3}, \boldsymbol{x}_{4}\}$
and let the distance matrix be such that
$\|\boldsymbol{x}_i-\boldsymbol{x}_j\|$ are given by:
$$
\vect{D}=\left[\begin{array}{rrrr}
0.0& 0.4& 0.6& 0.9 \\
0.4& 0.0& 0.9& 0.6 \\
0.6& 0.9& 0.0& 0.7 \\
0.9& 0.6& 0.7& 0.0  \\
       \end{array}
\right].
$$

\noindent
Then the agglomerative clustering algorithm yields what follows.

\begin{enumerate}
\item Start with 4 singletons $\big\{ \{\boldsymbol{x}_{1}\}, \{\boldsymbol{x}_{2}\}, \{\boldsymbol{x}_{3}\}, \{\boldsymbol{x}_{4}\} \big\}$;

the closest pair of clusters is such that
$d(\{\boldsymbol{x}_{1}\},\{\boldsymbol{x}_{2}\})=0.4$;

merge this pair.

\item Current partition is $\big\{ \{\boldsymbol{x}_{1},\boldsymbol{x}_{2}\}, \{\boldsymbol{x}_{3}\},\{\boldsymbol{x}_{4}\} \big\}$;

now the closest pair of clusters has $d(\{\boldsymbol{x}_{3}\},\{\boldsymbol{x}_{4}\})=0.7$;

merge them.

\item Current partition is $\big\{ \{\boldsymbol{x}_{1},\boldsymbol{x}_{2}\},\{\boldsymbol{x}_{3},\boldsymbol{x}_{4}\} \big\}$;

the closest (and the only) pair now fulfils $d(\{\boldsymbol{x}_{1},\boldsymbol{x}_{2}\},\{\boldsymbol{x}_{3},\boldsymbol{x}_{4}\})=0.6$ (which is a smaller distance than in the previous iteration);

merge this pair.
\end{enumerate}
Thus, when depicted on a dendrogram, we would get an inversion.
\end{remark}

\bigskip
Unfortunately, for OWA-based linkages, $d(C_z^{(j)}, C_u^{(j-1)}\cup C_v^{(j-1)})$
does not depend on $d(C_u^{(j-1)}, C_v^{(j-1)})$.
Therefore, we need to consider a different condition that will be suitable
in the dendrogram plotting context.

By the construction of the agglomerative hierarchical clustering algorithm,
we know that
$d(C_u^{(j-1)}, C_v^{(j-1)}) \le d(C_z^{(j-1)}, C_u^{(j-1)})$
and
$d(C_u^{(j-1)}, C_v^{(j-1)}) \le d(C_z^{(j-1)}, C_v^{(j-1)})$
for all $z\not\in\{u,v\}$, i.e.,
\[
d(C_u^{(j-1)}, C_v^{(j-1)})\le \min\{
d(C_z^{(j-1)}, C_u^{(j-1)}),
 d(C_z^{(j-1)}, C_v^{(j-1)})
\},
\]

\noindent
as otherwise, $C_z$ would have been merged with either $C_u$ or $C_j$
in the $j$-th step, and not $C_u$ with $C_v$.

Therefore, if:
\begin{equation}
d(C_z^{(j)}, C_u^{(j-1)}\cup C_v^{(j-1)}) \ge
\min\{
d(C_z^{(j-1)}, C_u^{(j-1)}),
 d(C_z^{(j-1)}, C_v^{(j-1)})
\},
\end{equation}

\noindent
then \textit{necessarily} the corresponding $h_d$
is increasing.

\bigskip
Focusing on the first type of extended OWA, we shall thus be interested in identifying
coefficient vectors $\boldsymbol{c}$ such that:
\begin{equation} \label{eq:owa1ge}
\mathrm{OWA}_{\boldsymbol{c}}(\boldsymbol{u}, \boldsymbol{v})
\ge \min\{
\mathrm{OWA}_{\boldsymbol{c}}(\boldsymbol{u}),
\mathrm{OWA}_{\boldsymbol{c}}(\boldsymbol{v})
\}
\end{equation}

\noindent
for all $\boldsymbol{u}, \boldsymbol{v}$ of any cardinalities,
where $(\boldsymbol{u}, \boldsymbol{v})$ denotes vector concatenation.

The first result we present toward this will help us establish some necessary conditions.

\begin{theorem}\label{thm:nec}
Let $\boldsymbol{c}=(c_1,c_2,\dots)$ be a  coefficient vector
with $c_1=1$ and $c_2,c_3,\dots\ge 0$.
Then for any $n$, $m$ and $u_1,\dots,u_n,v_1,\dots,v_m$, if it holds:
\begin{equation*}
\mathrm{OWA}_{\boldsymbol{c}}(u_1,\dots,u_n, v_1,\dots,v_m)
\ge \min\{
\mathrm{OWA}_{\boldsymbol{c}}(u_1,\dots,u_n),
\mathrm{OWA}_{\boldsymbol{c}}(v_1,\dots,v_m)
\},
\end{equation*}
then for all $k,l$ such that:
${\sum_{i=1}^n c_i}{\sum_{i=1}^l c_i} \ge {\sum_{i=1}^m c_i}{\sum_{i=1}^k c_i}$,
we must have:
\begin{equation} \label{eq:nec1}
    \sum_{i=k+1}^{k+l} c_i \sum_{i=1}^m c_i \ge \sum_{i=n+1}^{n+m} c_i \sum_{i=1}^l c_i.
\end{equation}
\end{theorem}

\begin{proof}
For any $k$ and $n$, with $k \leq n$,  let $\bar{\boldsymbol{e}}_k^n=(\bar{e}_1,\dots,\bar{e}_n)$
denote a sequence such that
$\bar{e}_1=\dots=\bar{e}_k = \frac{\sum_{i=1}^n c_i}{\sum_{i=1}^k c_i}$ and
$\bar{e}_{k+1}=\dots=\bar{e}_n = 0$.
It follows that
$\mathrm{OWA}_{\boldsymbol{c}}(\bar{\boldsymbol{e}}_k^n)=1$.

For any $k\le n$, $l\le m$, consider
$\boldsymbol{u}=\bar{\boldsymbol{e}}_k^n$ and $\boldsymbol{v}=\bar{\boldsymbol{e}}_l^m$.
As a necessary condition, we must have:

$$\mathrm{OWA}_{\boldsymbol{c}}(\boldsymbol{u}, \boldsymbol{v})
\ge 1= \min\{\mathrm{OWA}_{\boldsymbol{c}}(\boldsymbol{u}), \mathrm{OWA}_{\boldsymbol{c}}(\boldsymbol{v})\}.$$

If $u_1 \ge v_1$, i.e., if
${\sum_{i=1}^n c_i}{\sum_{i=1}^l c_i} \ge {\sum_{i=1}^m c_i}{\sum_{i=1}^k c_i}$,
the above is equivalent to stating that:

\begin{eqnarray*}
\frac{\sum_{i=1}^k c_i u_i + \sum_{i=k+1}^{k+l} c_i v_{i-k}}{\sum_{i=1}^{n+m} c_i} &\ge& 1 \qquad \Longleftrightarrow\\
\sum_{i=1}^k c_i u_i + \sum_{i=k+1}^{k+l} c_i v_{i-k} &\ge& \sum_{i=1}^{n+m} c_i  \qquad \Longleftrightarrow\\
\sum_{i=1}^k c_i \frac{\sum_{i=1}^n c_i}{\sum_{i=1}^k c_i} + \sum_{i=k+1}^{k+l} c_i \frac{\sum_{i=1}^m c_i}{\sum_{i=1}^l c_i} &\ge& \sum_{i=1}^{n+m} c_i  \qquad \Longleftrightarrow\\
\sum_{i=1}^n c_i \sum_{i=1}^l c_i + \sum_{i=k+1}^{k+l} c_i \sum_{i=1}^m c_i &\ge& \sum_{i=1}^{n+m} c_i \sum_{i=1}^l c_i  \qquad \Longleftrightarrow\\
\sum_{i=k+1}^{k+l} c_i \sum_{i=1}^m c_i &\ge& \sum_{i=n+1}^{n+m} c_i \sum_{i=1}^l c_i.  \\
\end{eqnarray*}

\end{proof}

\noindent
Particular cases of the above allow us to establish the following.

 \begin{corollary}
For $c_1 =1$, Eq.~\eqref{eq:owa1ge} necessarily implies that:
\begin{equation}\label{eq:neces1}
c_2 \ge c_3 \ge c_4 \ge \dots \ge 0.
\end{equation}
and:
\begin{equation}\label{eq:neces2}
\displaystyle\frac{\sum_{i=1}^{l} c_i }{\sum_{i=1}^{m} c_i}
\le
\displaystyle\frac{\sum_{i=l+1}^{2l} c_i}{\sum_{i=m+1}^{2m} c_i}, \qquad l \leq m.
\end{equation}

\end{corollary}

Equation~\ref{eq:neces1} is obtained from the condition in Theorem \ref{thm:nec}
 for $n\ge 2$, $l=m=1$, $k=n-1$, while Eq.~\ref{eq:neces2} follows from the special case $k=l$ and $n=m$.

While these two conditions are necessary, the following example illustrates that they are not sufficient.

\begin{example}
The sequence $\boldsymbol c = (1,1/2,3/8,3/8,9/32,9/32,9/32,9/32)$ satisfies the necessary conditions in Eq.~\eqref{eq:neces1} and Eq.~\eqref{eq:neces2}, however fails to satisfy Eq.~\eqref{eq:owa1ge} for e.g., $\boldsymbol u = (1.875,0,0)$, $\boldsymbol v = (1.6875, 1.6875,0,0,0)$.  We have $\mathrm{OWA}_{\boldsymbol c}(\boldsymbol u) = 1$, $\mathrm{OWA}_{\boldsymbol c}(\boldsymbol v) = 1,$ and $\mathrm{OWA}_{\boldsymbol c}(\boldsymbol u,\boldsymbol v) = 0.99306$.
\end{example}

Hence, we look towards establishing a sufficient condition, which we formulate as follows.

\begin{theorem} \label{thm:suf}
Let $\boldsymbol{c}=(c_1,c_2,\dots)$ be a  coefficient vector
with $c_1=1$ and $c_2 \ge c_3\ge \ldots \ge 0$.
Then a sufficient condition for \eqref{eq:owa1ge} is:
\begin{equation} \label{eq:suf1}
    \sum_{i=n+1}^{n+l} c_i \sum_{i=1}^m c_i \ge \sum_{i=n+1}^{n+m} c_i \sum_{i=1}^l c_i.
\end{equation}
\noindent
for any $n$, $l<m$ and $u_1,\dots,u_n,v_1,\dots,v_m$.
\end{theorem}

The proof is based on several lemmas presented below.

\begin{lemma} \label{lem1}
    Let $\mathbf c, \mathbf d$ be weighting vectors and $\mathbf u \in [0,\infty)^n$. Then $\mathrm{OWA}_{\mathbf c}(\mathbf u) \leq \mathrm{OWA}_{\mathbf d}(\mathbf u) $ if and only if, for $k = 1,\ldots,n$, $\mathrm{OWA}_{\mathbf c}(\mathbf e_k) \leq \mathrm{OWA}_{\mathbf d}(\mathbf e_k) $, where
    $\mathbf e_k=(1,1,\ldots,1,0,0\ldots,0)$ is the vector with $k$ ones and $n-k$ zeros.
\end{lemma}
\begin{proof}
It is sufficient to show this for $\mathbf u\in\{\mathbf x \in [0,1]: x_1 \geq x_2 \geq \ldots \geq x_n \}$ using symmetry and homogeneity of OWA functions. The graph of OWA on that subdomain is a fragment of a plane, and the inequalities at all the vertices are necessary and sufficient for one graph dominating the other.
\end{proof}

{Note that the above relationship can also be stated in terms of the cumulative sum of OWA weights, and related to the idea of stochastic dominance, i.e., for $\sum\limits_{i = 1}^n c_i = \sum\limits_{i = 1}^n d_i$ it holds that:
\begin{equation*}
\sum\limits_{i = 1}^k c_i \leq \sum\limits_{i = 1}^k d_i, \quad k = 1,\ldots,n.
\end{equation*}

}

Let us introduce notation $\mathbf{c}_{+k}=(c_{k+1}, c_{k+2},\ldots,c_{k+n+1})$.
\begin{lemma} \label{lem2}
 If a weighting vector $\mathbf c$ satisfies:
 \begin{equation} \label{eq:suf1b}
 {\sum_{i=k+1}^{k+m} c_i}{\sum_{i=1}^l c_i} \le {\sum_{i=1}^m c_i}{\sum_{i=k+1}^{k+l} c_i},
  \end{equation}
  then $\mathrm{OWA}_{\mathbf c_{+0}}(\mathbf u) \leq \mathrm{OWA}_{\mathbf c_{+k}}(\mathbf u) $
for all $\mathbf u$ and $k=1,2,\ldots$.
\end{lemma}

\begin{proof}[Sketch of the proof.]
Apply Lemma \ref{lem1} with $\mathbf c=\mathbf c_{+0}$ and $\mathbf d= \mathbf c_{+l}$ and note that $\mathrm{OWA}_{\mathbf c_{+k}}(\mathbf e_l) =\sum_{i=k+1}^{k+l} c_i / \sum_{i=k+1}^{k+m} c_i $. The trivial cases of the sums in \eqref{eq:suf1b} being 0 are considered separately.
\end{proof}

Another piece of notation, let
$\widetilde{\mathrm{OWA}}_{\mathbf c}(\mathbf u,\mathbf v)=\frac{1}{\sum_{i=1}^{n+m} c_i}(\sum_{i=1}^n c_i u_i + \sum_{i=n+1}^{n+m} c_i v_{i-n})$, whose difference to $\mathrm{OWA}$ is that no sorting step after concatenation takes place (it is assumed that $\mathbf u, \mathbf v$ are sorted separately).
\begin{lemma} \label{lem3}
    Let $\bar{\mathbf u}=(\bar u, \bar u,\ldots,\bar u) \in [0,\infty)^m$ and $\bar u=\mathrm{OWA}_{\mathbf c}(\mathbf u)$. Then for any $m\geq 1$:
    \begin{equation} \label{eq:lem:3}
        \widetilde{\mathrm{OWA}}_{\mathbf c}(\mathbf u,\bar{ \mathbf u})=\bar u.
    \end{equation}
\end{lemma}
\begin{proof}
{By definition, $\widetilde{\mathrm{OWA}}_{\mathbf c}(\boldsymbol u,\bar{ \boldsymbol u}) = \alpha \cdot \mathrm{OWA}_{\boldsymbol c}(\boldsymbol u) + (1 - \alpha)\cdot \mathrm{OWA}_{\boldsymbol c_{+m}}(\boldsymbol u) $, with $\alpha = \sum_{i = 1}^m c_i / \sum_{i = 1}^{m+n} c_i$.  Hence, due to idempotency of the OWA and since we are taking a convex combination, we have:
$$
\widetilde{\mathrm{OWA}}_{\mathbf c}(\boldsymbol u,{ \bar{\boldsymbol u}}) =
\alpha \cdot \bar u + (1-\alpha) \bar u = \bar u.
$$
}

\end{proof}

We can now formulate the proof of the above theorem.

\begin{proof}[Proof of Theorem \ref{thm:suf}]
Assume without loss in generality that in \eqref{eq:owa1ge},
$\mathrm{OWA}_{\boldsymbol{c}}(\mathbf u)=
\mathrm{OWA}_{\boldsymbol{c}}(\mathbf v) = \bar u$, as any increase in, say, $\mathbf v$ will lead to an increase on the left but not on the right, as well as $u_1\ge v_1$.
Then:
\begin{align*}
    \bar u=& \widetilde{\mathrm{OWA}}_{\mathbf c}(\mathbf u,\bar{ \mathbf u}) \quad \mbox{(Lemma \ref{lem3})}\\
    \le &\widetilde{\mathrm{OWA}}_{\mathbf c}(\mathbf u,\mathbf v)\quad \mbox{(Lemma \ref{lem2})}\\
    \le &{\mathrm{OWA}}_{\mathbf c}(\mathbf u,{ \mathbf v}) \quad \mbox{(decreasing weights starting from the second one)}.
\end{align*}
Note that $v_1$ will not end up in the first position in the sorted list $(\mathbf u, \mathbf v)$, hence even if $c_1<c_{n+1}$, it does not negate  the last inequality.
\end{proof}

Note that, as long as there are no zeroes in the denominators, we can write our sufficient condition as:
\begin{equation} \label{eq:suf2}
 \frac{\sum_{i=1}^l c_i} {\sum_{i=1}^m c_i} \le \frac{\sum_{i=k+1}^{k+l} c_i}{\sum_{i=k+1}^{k+m} c_i}
\end{equation}

\noindent
for all $k>1$, $m>l$. {This amounts to all effective weighting vectors obtained from the first $m$ arguments of $\boldsymbol c$ being stochastically dominated by all weighting vectors of the same length starting from a different index.}

Note two differences to the necessary conditions in Theorem \ref{thm:nec}.
First, there is a condition attached to \eqref{eq:nec1}. Second, we have $k<n$. Thus, the condition in \eqref{eq:suf2} is unfortunately stronger than the necessary conditions previously established.
We can weaken the sufficient conditions \eqref{eq:suf2} by modifying Theorem \ref{thm:suf} to:
\begin{equation} \label{eq:suf3}
 \frac{\sum_{i=1}^l c_i} {\sum_{i=1}^m c_i} \le \frac{c_k+\sum_{i=n+2}^{n+l} c_i}{\sum_{i=n+1}^{n+m} c_i},
\end{equation}

\noindent
for some $k<n$ (owing to the fact that $v_1$ will be positioned ahead of some $u_i, i=k,\ldots,n$ in the sorted list $(\mathbf u, \mathbf v)$).

\begin{example}\label{ex:maxandav}
It is easily seen that $\boldsymbol{c}=(1,0,0,\dots)$
and
$\boldsymbol{c}=(1,1,1,\dots)$,
corresponding to complete and average linkage,
respectively, fulfil the conditions
in Theorem~\ref{thm:suf}.
\end{example}

\begin{example}\label{ex:evecs}
All sequences of the form $\boldsymbol c = \boldsymbol e_k$ satisfy \eqref{eq:suf2}.
\end{example}

It is worthy to note that, in line with the condition that $\boldsymbol c$ need only be decreasing from $c_2$ onwards, sequences such as $(1,2,1,1,0,0,\ldots,0)$ and $(1,2,2,1,0,\ldots,0)$ can be verified as satisfying the sufficient condition.  One can also see that $(0,1,0,0,\ldots,0)$ will satisfy \eqref{eq:owa1ge}.  It can either be viewed separately from the framework that fixes $c_1 = 1$ or as a limiting case ($c_2 \to \infty)$.  In other words, general decreasingness is not a requirement on $\boldsymbol c$.

\medskip
On the other hand, we observe that a fairly simple rule that will ensure satisfaction of \eqref{eq:suf2} is for ratios between sequential values to be increasing, i.e., $c_i/c_{i+1} \leq c_{i+1}/c_{i+2}$. However, this is stronger than the sufficient condition.  While Examples \ref{ex:maxandav} and \ref{ex:evecs} above adhere to this rule, the following example shows that decreasing ratios also cannot be framed as a necessary requirement.

\begin{example}
A sequence that satisfies Theorem \ref{thm:suf}  is $\boldsymbol{c} = (1,1/2,1/5, 7/75,0,0,\ldots)$, however this does not have an increasing sequence of ratios, i.e. we have $c_1/c_2 = 2, c_2/c_3 = 5/2,$ however $c_3/c_4 = 15/7 < 5/2$.
\end{example}

Thus, the degree of the allowed tightening of the necessary and loosening
of the sufficient conditions is still an open problem.

\section{Conclusion and Future Work}\label{sec:conclusion}

OWA-based linkages were proposed in \cite{Yager2000:OWAlinkage}
and were reinvented in \cite{NasibovKandemirCavas2011:OWAlinkage}.
In \cite{CenaGagolewski2020:genieowa}, the practical usefulness of OWA-based
clustering was evaluated thoroughly on numerous benchmark datasets from
the suite described in \cite{Gagolewski2022:clustering-benchmarks}.
It was noted that adding the Genie correction for cluster size inequality
\cite{GagolewskiETAL2016:genie} leads to high-quality partitions,
especially based on linkages that rely on a few closest point pairs
(e.g., the single linkages and fuzzified/smoothened minimum).
These papers provide many examples of practically useful OWA weight generators.

In this paper, we have presented some previously missing theoretical
results concerning the OWA-based linkages. First, we have shown that the OWA-based
linkages and the Lance--Williams formula
only have three instances in common: the single, average, and complete linkages.
Both classes enable a very fast (linear-time) update between iterations
and thus are of potential practical interest.

Then, we gave some necessary and sufficient conditions for
the coefficient generating schemes to guarantee the resulting
dendrograms being free from unaesthetic inversions
(note that the mentioned Genie correction might additionally
introduce inversions by itself).
How to tighten these condition sets in the form of ``if and only if''
statements is still an open problem: follow-up research is welcome.

In the future, we suggest considering similar results concerning
different weighting triangle generating schemes,
e.g.,
$w_{i, z}=w\left(\frac{i}{z}\right)-{w}\left(\frac{i-1}{z}\right)$,
where $w:[0,1]\to[0,1]$ is a monotone bijection
(compare \cite{Yager1988:owa}).

Furthermore, in \cite{CenaGagolewski2020:genieowa}, a generalised,
three-stage OWA linkage scheme was introduced.
There are also generalisations of the Lance--Williams formula,
e.g., \cite{GanMaWu2007:dataclustering}. Inspecting the relationships
between them could also be conveyed.

\section*{Acknowledgments}
\noindent
This research was supported by the Australian Research Council Discovery
Project ARC DP210100227 (MG, GB, SJ).

\section*{Conflict of interest}

All authors certify that they have no affiliations with nor involvement in any
organisation or entity with any financial interest or non-financial interest
in the subject matter or materials discussed in this manuscript.

\clearpage

\printcredits  %

\end{document}